\documentclass{article}
\usepackage{amssymb, amsmath, amsthm}
\usepackage{xcolor}
\usepackage{booktabs}
\usepackage{multirow}
\usepackage{graphicx}
\usepackage{subcaption} 
\usepackage{enumitem}
\usepackage[compress]{cite}
\usepackage[final]{corl_2026}

\newcommand{\name}{SlipSense}

\ifdefined\showrevisions
  \newcommand{\revised}[1]{\textcolor{blue}{#1}}
\else
  \newcommand{\revised}[1]{#1}
\fi

\newtheorem{proposition}{Proposition}
\newtheorem{corollary}{Corollary}

\title{SlipSense: Multimodal Tactile Learning for Low-Latency and Generalized Slip Detection}

\author{
  \textbf{Tong Jian \quad Aditya Thurvas Senthil Kumar \quad  Xinyi Li \quad Ziling Chen \quad Tianyu Dai}\\
  \textbf{Ali Sengul \quad Matteo Grimaldi \quad Wenjie Lu \quad Saleh Nabi \quad Tao Yu}\\ \\
  {\normalfont Dexterous AI Group, Analog Devices, Inc.}\\
  {\normalfont\texttt{\{Tong.Jian, Aditya.Kumar, Xinyi.Li, Archer.Chen, Tianyu.Dai, }}\\
  {\normalfont\texttt{\{Ali.Sengul, Matteo.Grimaldi, Wenjie.Lu, Saleh.Nabi, Tao.Yu\}@analog.com}}
}

\hypersetup{
  pdftitle={SlipSense: Multimodal Tactile Learning for Low-Latency and Generalized Slip Detection},
  pdfauthor={Tong Jian, Aditya Thurvas Senthil Kumar, Xinyi Li, Ziling Chen, Tianyu Dai}\\
  \textbf{Ali Sengul, Matteo Grimaldi, Wenjie Lu, Saleh Nabi, Tao Yu}}

\begin{document}
\maketitle

\begin{figure}[hbt!]
  \centering
  \includegraphics[width=\linewidth]{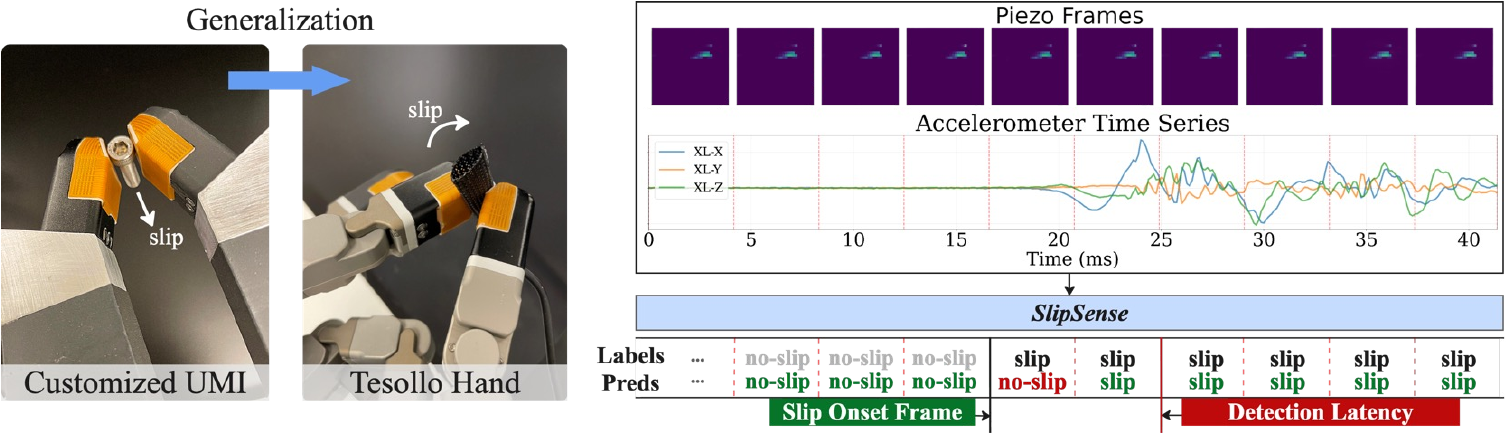}
  \caption{
    \textit{Left:} TacV5 multimodal tactile sensor on UMI gripper and Tesollo hand, combining high-density pressure sensing (Piezo: $32{\times}32$, 240\,Hz, 155\,taxels/cm²) with high-bandwidth vibration sensing (XL: 3-axis MEMS accelerometer, 8\,kHz). \textit{Right:} We introduce \textit{\name}, a multimodal framework that fuses Piezo and XL signals for accurate, low-latency slip detection, achieving cross-object and cross-platform generalization. \revised{Project page coming soon.}}
  \label{fig:highlight}
\end{figure}

\begin{abstract}
Slip detection is fundamental to dexterous manipulation, yet existing systems often lack precise detection latency characterization and cross-platform generalization. We present \name, a multimodal tactile slip detection framework built on TacV5, a compact sensor integrating a $32{\times}32$ piezoresistive array (240\,Hz) and a 3-axis MEMS accelerometer (8\,kHz). The piezoresistive array captures spatial pressure distribution while the accelerometer captures friction-induced vibration, providing complementary slip cues. The framework performs modality-specific encoding, intra-sensor fusion, and cross-modal attention with causal temporal prediction at 240\,Hz. Experiments on a 1.4M-frame dataset spanning 37 objects demonstrate that the two modalities are crucially complementary: it achieves 96.7\% Macro F1 with false-positive rate below 1.6\%, detecting 76\% of slip events within 23.1\,ms. Remarkably, trained purely on a UMI, ~\name~generalizes zero-shot to a Tesollo dexterous hand, transferring across unseen objects, distinct sensor units and new robotic platforms without any retraining.\end{abstract}


\section{Introduction}\label{sec:introduction}

Slip detection is fundamental to dexterous manipulation, as slip is often one of the earliest observable indications of grasp instability. Early detection enables corrective action before object loss. In practice, slip rarely occurs abruptly; instead, it emerges progressively, spanning a continuum from localized partial motion (\textit{incipient slip}) to full uncontrolled sliding (\textit{gross slip}). Even subtle relative motion can destabilize a grasp and propagate into failure, making timely detection critical for robust manipulation. The entire slip process is therefore relevant to grasp stability, with particular emphasis on slip onset, referred to as the first occurrence of relative motion at the contact interface.

Slip generates diverse physical signatures at the contact interface, including pressure redistribution, tangential force variation, friction-induced vibration, acoustic emission, and even thermal changes. Each sensing modality captures a partial view through its own distinct transduction mechanisms. As a result, modalities provide complementary cues but distinct failure modes, making multimodal sensing a promising direction for robust slip detection.
Despite decades of progress since early accelerometer-based artificial skin systems in the 1980s~\cite{howe1989}, two key challenges still remain unresolved today~\cite{zhang2026slipsensors}. \textit{\textbf{(i) Precise detection latency}} remains poorly characterized. Detection latency measures the interval between slip onset and the first correct detection, capturing the time needed to accumulate sufficient evidence. This delay often dominates system response time beyond inference speed, yet is rarely reported. Latency also trades off inherently against false alarms. Increasing sensitivity reduces detection delay but raises false positives, and prior work rarely evaluates this trade-off. \textit{\textbf{(ii) Generalization}} across unseen objects, sensor units, and robotic platforms remains limited. Most systems are evaluated under conditions closely matching their development setup, leaving it unclear how well they generalize to novel scenarios. This includes shifts in object properties, sensor-level variations, and robotic systems with different contact geometry and actuations.





To address these challenges, we present \textbf{\textit{\name}}, built on a novel multimodal tactile sensor module (TacV5) integrating a high-density piezoresistive array (Piezo, 240\,Hz, $32{\times}32$, 155\,taxels/cm²) and a high-bandwidth 3-axis MEMS accelerometer (XL, 8\,kHz). These two modalities capture complementary physical signatures of contact. Piezo measures a spatial normal pressure distribution and its temporal evolution across the contact surface, while XL captures friction-induced vibrations arising from relative motion at the interface. Crucially, their failure modes are complementary rather than shared: spatial sensing is susceptible to false positives from non-slip pressure changes such as grasp variation, which produce vibration signatures distinct from slip; whereas vibration sensing can be confounded by environmental disturbances that do not introduce quasi-static force redistribution. 
Jointly processing both streams mitigates each modality's dominant failure mode. Modalities are synchronized across all fingertip sensors, encoded by dedicated backbones, and fused through a causal temporal classifier that performs ternary-class prediction (no-contact, no-slip, slip) at 240\,Hz.

Our contributions are:
\textit{\textbf{(1)~A multimodal learning framework}} demonstrating that Piezo and XL signals are complementary and jointly crucial, with fusion achieving 96.7\% Macro F1 vs. 81--84\% for either modality alone, while reducing false-positive rate below 1.6\%.
\textit{\textbf{(2)~A fast slip detection system}} operating at 240\,Hz, where 76\% of slip events are detected within 23.1\,ms, combining detection latency and model inference.
\textit{\textbf{(3)~Cross-domain generalization.}} We conduct experiments on \textit{1.4M} labeled frames covering \textit{37} diverse objects, demonstrating generalization to unseen objects, distinct sensor instances, and zero-shot transfer from a UMI parallel-jaw gripper to a Tesollo dexterous hand\revised{, spanning unseen articulation dynamics and system noise without retraining}.



The remainder of the paper is organized as follows. Section~\ref{sec:related_work} reviews related work across tactile sensing modalities. Section~\ref{sec:system} describes the sensor hardware and data collection setup. Section~\ref{sec:methodology} presents the multimodal fusion framework. Section~\ref{sec:experiments} reports experimental results demonstrating high-accuracy, low-latency slip detection and generalization across scenarios.
\section{Related Work}\label{sec:related_work}

Tactile slip detection relies on diverse sensing principles, each with characteristic strengths and failure modes. Optical sensors such as GelSlim~\cite{dong2018gelslim}, GelSight Mini~\cite{hu2024gelsightmini}, and GelStereo~\cite{cui2024gelstereo} infer slip from displacement fields at high spatial resolution, achieving upto 95\% accuracy but at 20--60\,Hz, with degradation on unseen objects~\cite{hu2024gelsightmini} and smooth surfaces~\cite{dong2018gelslim,cui2024gelstereo}. Among electrical sensors, shear force plays a central role to slip detection through the friction ratio. Capacitive and piezoelectric sensors can directly capture shear related signals, though the latter is limited to dynamic transients~\cite{romeo2020survey}. Piezoresistive arrays measure normal force and often at low spatial resolution. These constraints weaken spatial contact fidelity, leading most prior work to collapse 2D spatial structure into 1D temporal signals for frequency-domain analysis~\cite{schurmann2010,romeo2017piezoresistive}. Our TacV5 achieves a $0.7\times0.92$\,mm pitch ($38.8{\times}$ higher spatial density than~\cite{schurmann2010} and $6.2{\times}$ than FlexiTac~\citep{huang2025vtrefine}), enabling joint modeling of spatial pressure structure and its temporal evolution.
Vibration-based methods capture friction-induced oscillations~\cite{holweg1996slip} at high bandwidth but respond to environmental disturbances as well~\cite{cravetz2025}. These complementary strengths and weaknesses naturally motivate multimodal fusion. Several multimodal systems have been proposed~\cite{reinecke2014biotac,su2015biotac,xu2025multimodaltactilefingertipdesign,komeno2024incipientslipdetectionvibration,shepherd2025texture,8346698}, demonstrating benefits from combining modalities. However, many do not specifically target slip detection, and none evaluates the two critical factors emphasized here, detection latency and generalization.

Detection latency is the operationally relevant metric for closed-loop control, yet is rarely reported especially alongside false-positive rates. Among the few that measure it: BioTac benchmarks against an object-mounted IMU~\cite{su2015biotac}, Romeo et al.~\cite{Romeo2021AutomaticSlippage} report 76.7\% of trials below 30\,ms, ~\citet{ayral2026reactive} report 20.4\,ms across 20 trials on one single object, and Massalim et al.~\cite{massalim2020deep} report 17\,ms but validate on only three objects. Others report model execution time rather than detection latency~\cite{zhao2025universal}. Generalization is equally limited. Accuracy degrades on unseen objects~\cite{hu2024gelsightmini}, scaling to multifingered hands drops performance~\cite{james2021tmo}. Recent efforts advance generalization over grasp poses on a dexterous hand~\cite{zhao2025universal} and across gripper types ~\cite{cravetz2025}, but do not characterize detection latency alongside transfer performance. To the best of our knowledge, we are the first to jointly evaluate both.

\begin{figure}[t]
  \centering
\includegraphics[width=0.85\textwidth]{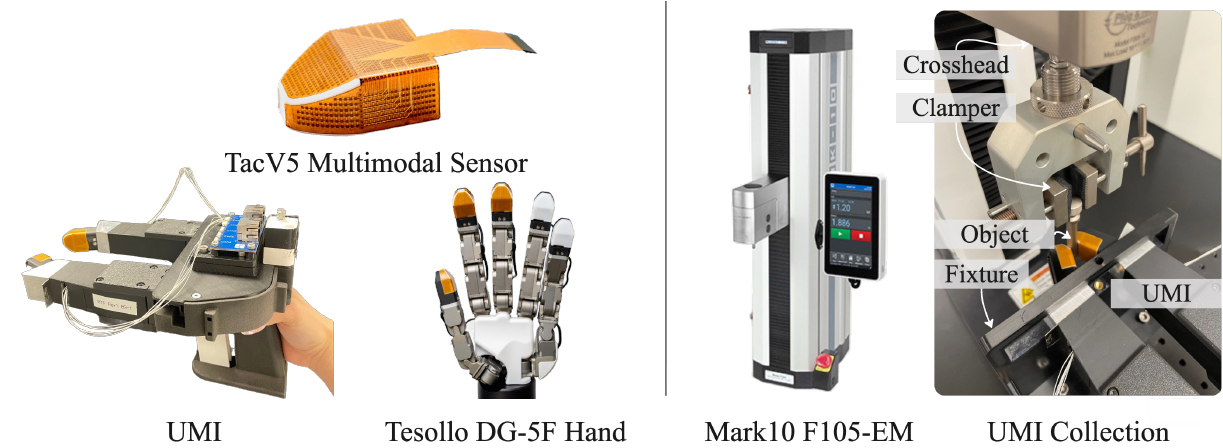}
  \caption{\textit{Left:} TacV5 multimodal sensors on UMI grippers and the Tesollo DG-5F dexterous hand at the thumb, index, and middle fingers. \textit{Right:} Data collection platform. The UMI is mounted on a separate table to physically isolate vibrations from the Mark10 and constrained by a fixture to suppress slip-induced shaking. Objects are secured with an adjustable clamp, while controlled slip is generated through programmable vertical crosshead motion. The operator controls only gripper opening and closing. No external camera is required.
}\label{fig:setup}
\end{figure}

\section{System and Data Acquisition}\label{sec:system}
\subsection{Sensor Specification}
We developed TacV5, a \revised{compact multimodal tactile sensing prototype}. 
The primary modality is a $32{\times}32$ piezoresistive array (Piezo) with 942 active taxels ($0.45\times0.45$\,mm elements at $0.7\times0.92$\,mm pitch, 155\,taxels/cm$^2$) on a flexible polyimide substrate, covering a 0--5\,MPa normal pressure range (i.e., 500\,N on $1\times1$\,cm$^2$ area) at 12 bit resolution (i.e., 4096 levels) and 240\,Hz. The sensor measures normal pressure only. 
The module also includes a 3-axis MEMS accelerometer (XL) sampling at 8\,kHz, capturing high-frequency vibration from frictional dynamics at the contact interface. Additional modalities (MEMS microphone, bone-conduction microphone) are included in the module but not used in this work. Lastly, the module communicates over CAN-FD and streams synchronized multimodal data at 240\,Hz; we refer to each synchronized sample as a \textit{frame}.

Compared to optical tactile sensors such as GelSight~\citep{yuan2017gelsight} and DIGIT~\citep{lambeta2020digit}, TacV5 offer a thinner form factor, greater mechanical robustness, and requires neither internal illumination nor cameras. Its compact design fits within a standard fingertip pad and remains mechanically independent of gripper actuation, enabling deployment across platforms without sensor-side modification. As shown in Fig.~\ref{fig:setup}, we integrate TacV5 on two platforms: a UMI~\cite{chi2024umi} parallel-jaw gripper with two sensor units (one per jaw), and a Tesollo DG-5F dexterous hand~\cite{tesollo_dg5fm} with sensor units to three fingers.
\begin{figure}[t]
  \centering
  \includegraphics[width=0.9\textwidth]{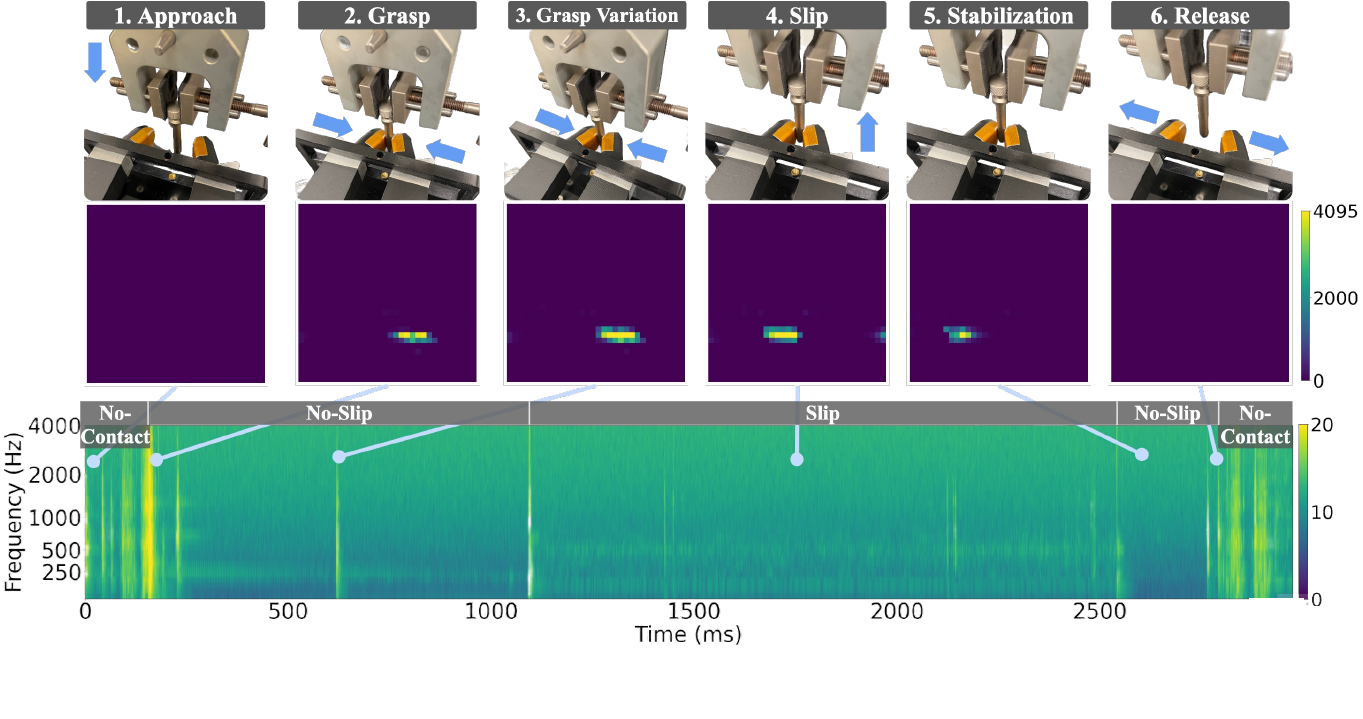}
  \caption{Data collection protocol. Each collection follows a six-step sequence, illustrated with UMI and one sensor's readings: (1)~crosshead descends and holds; (2)~gripper closes; (3)~operator varies grasp force to diversify contact conditions; (4)~crosshead ascends to induce slip; (5)~crosshead stops; (6)~gripper opens. Slip onset and ongoing slip produce distinct XL signatures, while non-slip events such as grasp variation (step~3) and gripper actuation (step~2$\&$6) also generate strong vibration responses. Disambiguating these cases requires spatial pressure cues from Piezo, motivating multimodal fusion.}
  \label{fig:system}
\end{figure}
\subsection{Data Collection and Labeling}\label{sec:data_collection}
A fundamental challenge in learning-based slip detection is obtaining reliable ground-truth labels. Existing methods infer slip onset from optical marker displacement~\cite{dong2018gelslim,james2018tactip}, friction based criteria on the sensor signal itself~\cite{Romeo2021AutomaticSlippage,higuera2024sparsh}, or human annotations~\cite{cui2024gelstereo, zhao2025universal}. Since labels are often derived from the sensor outputs or subjective observation, slip onset remains ambiguous near the transition boundary. Following the idea of directly measuring object displacement with linear encoders~\cite{massalim2020deep}, we adopt a Mark-10 F105-EM motorized test stand. As shown in Fig.~\ref{fig:setup}, the crosshead clamps one end of the object while the sensor-equipped gripper holds the other. For non-stretchable objects, crosshead displacement directly translates to object slip at the gripper surface. Crosshead position is recorded at 0.02\,mm resolution, enabling sub-millimeter slip labels independent of the sensor signal. \revised{The Mark-10 serves solely as a labeling instrument to provide ground-truth labels with high precision; it is absent at deployment.} The non-stretchable constraint applies only during data collection and does not affect deployment.

Each collection follows a fixed six-step sequence (Fig.~\ref{fig:system}). Slip occurs only during the crosshead ascent (step~4). Slip onset is defined when cumulative crosshead displacement exceeds 0.07\,mm\revised{, safely above the Mark-10 encoder resolution}. 
Slip termination is determined using the same criteria when the crosshead stops. All frames within this interval are labeled as \textit{Slip}. Frames with contact but outside this interval are labeled as \textit{No slip}, while frames are labeled as \textit{No contact} \revised{when the mean taxel value falls below one ADC count, a universal threshold shared across all objects and sensors without per-object calibration}. During slip, Piezo exhibits spatial pressure shift while XL shows sustained vibration patterns in the log-mel spectrogram. These patterns indicate that discriminative features are already present in the raw sensor streams before any learned encoding.



\begin{figure}[t]
\centering
  \includegraphics[width=0.9\linewidth]{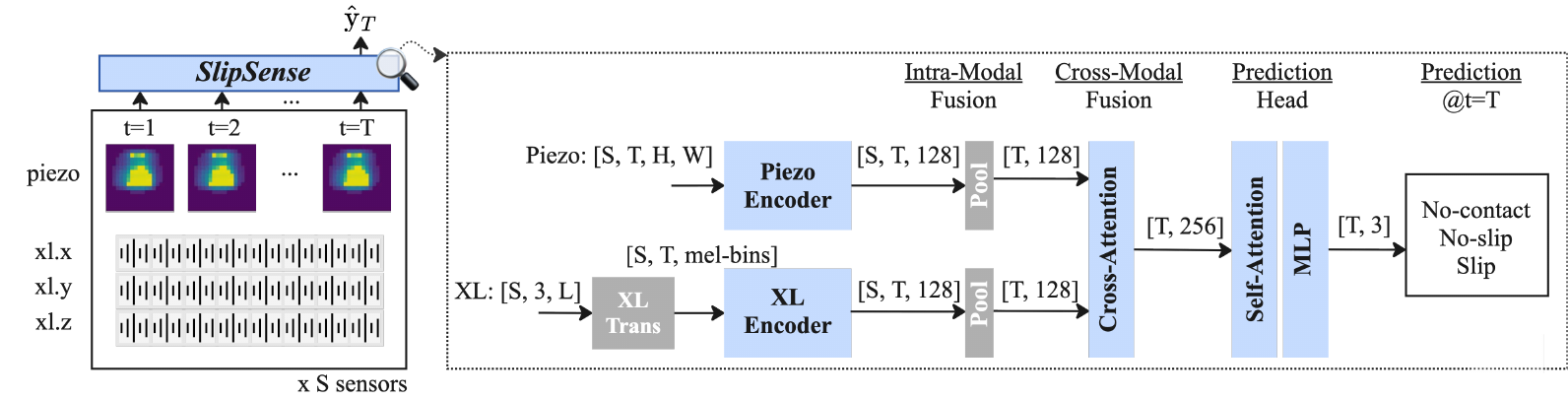}
  \caption{
    Overview of \name. Given an observation window of length $T$, Piezo and XL signals are independently encoded into modality specific embeddings at each frame, followed by intra-modal fusion and cross-modal fusion. A causal attention module with a prediction head outputs a 3-class prediction for the last frame at 240\,Hz.
  }
  \label{fig:arch}
\end{figure}

\section{Multimodal Learning}\label{sec:methodology}
The two sensing modalities, Piezo and XL, differ substantially in temporal granularity, spatial structure, and signal statistics, making naive fusion ineffective. In this section, we present a causal, learning-based \name \,framework (Fig.~\ref{fig:arch}). 

\subsection{Problem Formulation}
Let $S$ denote the number of sensor units and $T$ the observation window length in frames. At each frame $t$, sensor $s \in \{1, \ldots, S\}$ produces a Piezo tactile image $\mathbf{p}_{t,s} \in \mathbb{R}^{H \times W}$ and an XL sequence $\mathbf{a}_{t,s} \in \mathbb{R}^{3 \times L}$, acquired over the time interval $[t-1, t]$. The observation at frame $t$ is defined as $\mathcal{O}_t = \left\{\,\mathbf{p}_{s,t},\;\mathbf{a}_{s,t} \;\middle|\; s \in [S]\,\right\}.$

Each frame carries a label $y_t \in \mathcal{Y} = \{0, 1, 2\}$, corresponding to the \textit{no-contact}, \textit{no-slip}, and \textit{slip} states defined in Sec.~\ref{sec:data_collection}.
The objective is to learn a classifier $f_\theta$ that takes an observation sequence as input and predicts the label of its last frame $T$ alone:
\begin{equation}                                                \hat{y}_T = f_\theta(\mathcal{O}_{1:T}), \qquad \hat{y}_T \in \mathcal{Y}. \label{eq:task}                               \end{equation}

\subsection{\name~Framework}

\textbf{Piezo encoder.}
We adopt a lightweight convolutional architecture. The network consists of three sequential blocks, each with a $3\times3$ convolutions, ReLU activations, and $2\times2$ max pooling, followed by a fully connected layer that produces a Piezo embedding $\mathbf{z}^{p}\in\mathbb{R}^{d}$ with $d{=}128$. We also explored larger and pretrained alternatives (Appendix~\ref{app:encoder}), but found this compact design both more effective and better suited for low-latency inference.

\textbf{Accelerometer encoder.}
Accelerometer signals share the frequency-rich temporal structure of audio, motivating a spectral encoding approach. We apply a 10\,Hz high-pass filter\revised{, with the cutoff chosen empirically to remove gravity and gross gripper motion while preserving higher-frequency slip-induced features}. We then compute per-axis log-mel spectrograms from a causal 16\,ms analysis window aligned to each Piezo frame. The three per-axis spectrograms are aggregated via
log-sum-exp to obtain a slip-direction-invariant spectrogram (proof in Appendix~\ref{app:xl_invariance}). The merged spectrogram exhibits clear slip signatures visible in both low-frequency and high-frequency bands (Fig.~\ref{fig:system}). We encode it using SSAST-Tiny~\cite{gong2022ssast}, producing a per-frame embedding $\mathbf{z}^{xl} \in \mathbb{R}^{128}$. Transformation and architectural details are in Appendix~\ref{app:xl_encoder}.

\textbf{Intra and cross modality fusion.}
Platforms may carry different numbers of sensor units $S$. To handle this variable-cardinality setting, we aggregate per-sensor embeddings within each modality using element-wise max pooling, i.e., $\mathbf{z}_t^{p}=\max_{s\in[S]}\mathbf{z}_{t,s}^{p}$. This produces an embedding independent of sensor count, providing the potential for zero-shot transfer across different platform configurations. \revised{This design intentionally discards finger identity, making detection agnostic to finger count.} To integrate the Piezo and XL features, we apply multi-head cross-attention (MHAttn):
\begin{align}
  \mathbf{z}_t^{p\prime}  &= \mathrm{LayerNorm}\!\left(\mathbf{z}_t^{p}
    + \mathrm{MHAttn}(Q{=}\mathbf{z}_t^{p},\;  K{=}\mathbf{z}_{\leq t}^{xl},\; V{=}\mathbf{z}_{\leq t}^{xl})\right),  \label{eq:cross1} \\
  \mathbf{z}_t^{xl\prime} &= \mathrm{LayerNorm}\!\left(\mathbf{z}_t^{xl}
    + \mathrm{MHAttn}(Q{=}\mathbf{z}_t^{xl},\; K{=}\mathbf{z}_{\leq t}^{p},\;  V{=}\mathbf{z}_{\leq t}^{p})\right),  \label{eq:cross2}
\end{align}
The fused representation concatenates both updated embeddings,
$
  \mathbf{z}_t^{\mathrm{fused}} = \bigl[\mathbf{z}_t^{p\prime};\;\mathbf{z}_t^{xl\prime}\bigr]
  \in \mathbb{R}^{d_p + d_{xl}}.
$

\textbf{Temporal frame-based predictor.}
The fused embeddings $\{\mathbf{z}_{t}^{\mathrm{fused}}\}_{t=1}^{T}$ are processed by a causal self-attention layer that accumulates temporal context. A prediction head produces per-frame logits $\hat{y}_{t} \in \mathcal{Y}$. Training supervises only the last frame $T$, which has access to the full observation window:
\begin{equation}
  \mathcal{L} = \mathrm{CE}(\hat{y}_T,\, y_T).
  \label{eq:loss}
\end{equation}
During inference, the model can output prediction per incoming frame, enabling real-time operation at 240\,Hz when inference completes within one frame interval.


\begin{figure}[t]
  \centering
  \includegraphics[width=\linewidth]{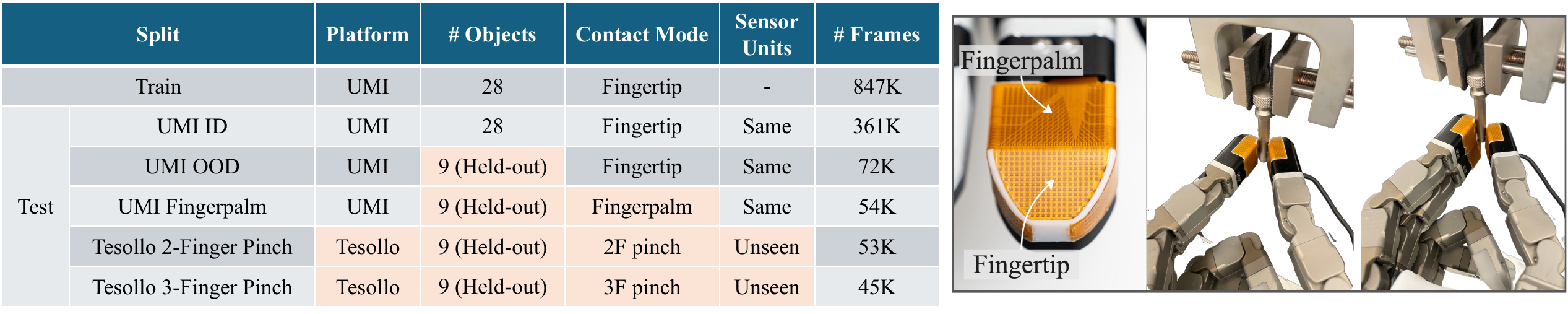}
  \caption{\textit{Left:} Dataset overview. Five test splits progressively evaluate generalization across held-out objects, unseen contact region, and new robotic platforms with physically distinct sensor units, evaluated using Tesollo 2-finger and 3-finger pinch \textit{(Right)}.
  }
  \label{fig:dataset}
\end{figure}

\section{Experiments}\label{sec:experiments}
We evaluate our framework from four perspectives: (1) benefits of multimodal fusion, (2) generalizability across unseen objects, contact regions, sensor units, and robotic platforms, (3) detection latency characterization, and (4) ablations on key design factors affecting performance.

\subsection{Experimental Setup}

\textbf{Datasets.}
We collected a dataset of over 1.4\,M frames from 37 objects spanning diverse shapes and surface materials (Appendix Fig.~\ref{fig:object}), with 28 objects used for training and 9 held out for evaluation. Data were collected under five slip speeds (2--18\,mm/s) and varied grasp forces (Appendix Fig.~\ref{fig:coverage}). To reduce artifacts and false positives, we additionally include control scenarios without induced slip (Appendix~\ref{app:dataset}). The resulting dataset contains 57.6\% \textit{no-contact}, 23.7\% \textit{slip}, and 18.8\% \textit{no-slip} frames. We evaluate using five test splits with progressively increasing distribution shift (Fig.~\ref{fig:dataset}).

\textbf{Implementation details.}
We use a default observation window of $T{=}60$ (250\,ms). Piezo frames use raw voltage outputs normalized by the 12 bit ADC range of 4096\revised{; no per-unit calibration is applied. Training uses 2 UMI sensor units, while evaluation spans 2 UMI and 3 Tesollo units over a 4-month period, covering unit-to-unit variation, drift, and aging}. Modality specific augmentation is applied during training. All experiments are repeated with 3 random seeds and reported as the mean. Full implementation details are provided in Appendix~\ref{app:implementation}.

\textbf{Evaluation metrics.}
We report: \textit{Detection latency}, the median delay from slip onset to first detected slip frame; \textit{False Positive Rate (FPR)}, the fraction of frames incorrectly predicted as slip; \textit{Accuracy}, the fraction of correctly classified frames; and \textit{Macro F1}, the equally weighted class averaged F1.

\textbf{Baselines.}
We compare against three analytic baselines: \textit{Piezo only}, adapted from pressure entropy methods~\cite{hu2024gelsightmini}; \textit{XL only}, using MFCC features and QDA~\cite{Yoo_2021}; and \textit{Piezo-XL}, combining both via AND logic. Full implementation details are provided in Appendix~\ref{app:baseline}.

\subsection{Analysis and Discussion}

\begin{table*}[t]
    \caption{Multimodal Fusion. Latency is reported as the mean of per-seed median detection latency (in frames) across 3 random seeds; baselines are analytic methods. Fusing Piezo and XL achieves the lowest latency and FPR while substantially improving Macro F1 over either modality alone. }
    \label{tab:modality_ablation}
    \begin{center}
    \resizebox{\linewidth}{!}{
    \centering
    \begin{tabular}{c cc| cc cc cc cc}
    \toprule
        \multirow{2}{*}{Method} &
        \multicolumn{2}{c|}{Modality} &
        \multicolumn{4}{c}{UMI ID} &
        \multicolumn{4}{c}{UMI OOD} \\
        \cmidrule(lr){2-3}
        \cmidrule(lr){4-7}\cmidrule(lr){8-11}
        & Piezo & XL
        & Lat (fr) $\downarrow$ & FPR (\%) $\downarrow$ & Acc (\%) $\uparrow$ & F1 (\%) $\uparrow$
        & Lat (fr) $\downarrow$ & FPR (\%) $\downarrow$ & Acc (\%) $\uparrow$ & F1 (\%) $\uparrow$ \\
    \midrule
        \multirow{3}{*}{Baselines} 
        & \checkmark &            & \textbf{1.00} & 21.69 & 72.87 & 74.51 & 2.00 & 21.08 & 73.84 & 75.29 \\
        &            & \checkmark & \textbf{1.00} & 22.98 & 79.41 & 81.79 & 2.00  & 19.31 & 82.99 & 84.55 \\
        & \checkmark & \checkmark & \textbf{1.00} & 9.28  & 81.16 & 79.96 & 2.00 & 7.40 & 83.33  & 82.18 \\
    \midrule
        \multirow{3}{*}{SlipSense}  & \checkmark & 
        & 20.50 &  7.54 & 84.63 & 83.54 & 36.17 &  8.43 & 82.69 & 81.91 \\
        &           & \checkmark 
        &  4.67 &  2.14 & 84.21 & 81.61 &  3.00 &  3.24 & 84.74 & 81.33 \\
        & \checkmark & \checkmark
        &  2.00 &  \textbf{1.33} & \textbf{96.98} & \textbf{96.77} &  \textbf{1.50} &  \textbf{1.57} & \textbf{95.99} & \textbf{95.75} \\
    \bottomrule
    \end{tabular}
    }
    \end{center}
    \vspace{-5pt}
\end{table*}

\begin{figure}[t]
  \centering
  \includegraphics[width=0.95\linewidth]{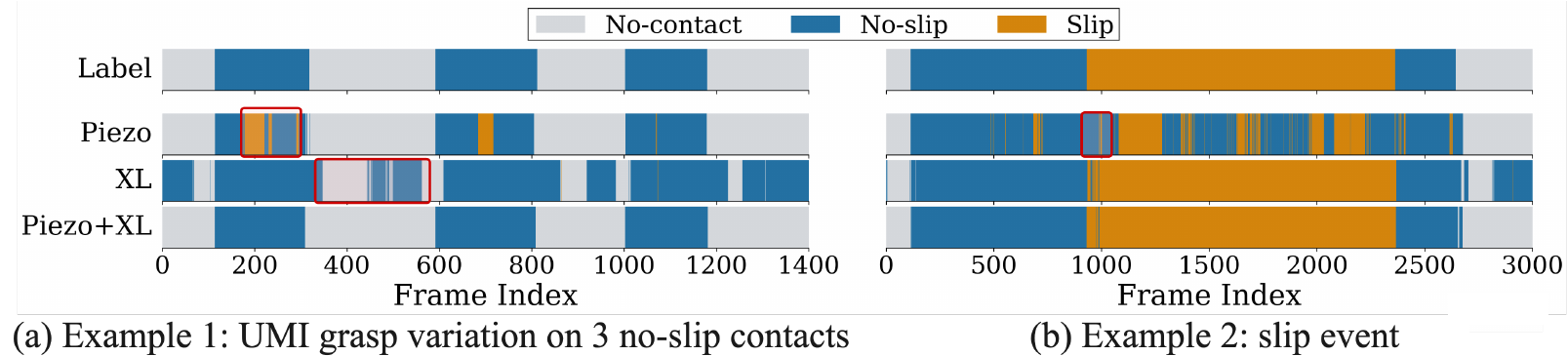}
  \caption{
    Prediction examples. Failure modes are highlighted in red: Piezo can confuse grasp variation with slip and delay slip onset detection, while XL is sensitive to gripper aftershocks during no contact. Fusion suppresses both failure modes, achieving low latency and low false positives.
  }
  \vspace{-5pt}
  \label{fig:example}
\end{figure}

\begin{table}[t!]
    \caption{
    Zero-shot generalization. One single model trained on UMI train set, evaluated zero-shot across three transfer axes: contact location, platform, and finger count.   
    }
    \label{tab:transferability}
    \begin{center}
    \resizebox{0.8\linewidth}{!}{
    \centering
    \begin{tabular}{lccccc}
    \toprule
        Transfer Scenario &
        Setup &
        Lat (fr) $\downarrow$ & FPR (\%) $\downarrow$ & Acc (\%) $\uparrow$ & F1 (\%) $\uparrow$ \\
    \midrule
        Contact location & UMI Finger Palm  & 3.00 & 1.99 & 94.83 & 94.56 \\
        Platform       & Tesollo 2-Finger & 3.00 & 1.14 & 95.89 & 94.39 \\
        Finger count   & Tesollo 3-Finger & 5.00 & 3.12 & 90.85 & 87.24 \\
    \bottomrule
    \end{tabular}}
    \end{center}
\end{table}

\textbf{Piezo and XL are crucially complementary.}
Table~\ref{tab:modality_ablation} shows an inherent sensitivity-specificity tradeoff in single-modality methods: analytic baselines achieve low latency but produce high FPR (19--22\%), while learned models reduce FPR at the cost of increased latency. \revised{Fig.~\ref{fig:example} illustrates how multimodal sensing can improve the tradeoff: the Piezo modality provides contact grounding but confuses grasp variation with slip and misses subtle slip cues, whereas the XL modality is sensitive to slip but lacks contact context and responds to environmental vibrations. Their failure modes are complementary, making fusion a natural path to improving both metrics simultaneously. Importantly, the fusion strategy matters. Analytic AND fusion reduces FPR but suppresses true slip detections significantly. In contrast, \name~achieves the best across all metrics, combining higher Macro F1 and lower FPR with sub-frame latency. Further comparisons show that the cross-attention module outperforms concatenation and late fusion alternatives (Appendix~\ref{app:ablation_fusion}), suggesting that fusion strategy contributes meaningfully to these gains.}


\textbf{Generalization.}
Unseen-object results (UMI OOD) in Table~\ref{tab:modality_ablation} show additional latency needed for baseline to maintain similar detection performance, indicating slightly more challenging test cases. Notably, our models show minor degradation, suggesting generalizability over object contact patterns.
Table~\ref{tab:transferability} further evaluates held-out objects under progressively harder conditions across contact regions, sensor units, and robotic platforms with multi-finger interactions. Fingertip-to-fingerpalm transfer maintains strong performance, indicating the learned features are not tied to a specific contact location. UMI-to-Tesollo presents a more challenging transfer setting \revised{due to articulation dynamics and system noise unseen during UMI training. Tesollo's direct joint actuation changes contact stiffness and introduces distinct vibration signatures in XL. The XL augmentation is designed to reduce sensitivity to platform-specific signatures while preserving slip-relevant features (Appendix~\ref{app:ablation_aug}).} Despite such differences, the 2-finger pinch achieves 94.39\% Macro F1. The 3-finger pinch further alters contact geometry, widening the distribution gap from training, yet performance remains reasonable at 87.24\% Macro F1. Across all settings, \name~generalizes without targeted fine-tuning, suggesting it captures slip signatures intrinsic to the contact interface rather than artifacts of specific platforms or sensors.


\textbf{Detection latency characterization.} 
Figure~\ref{fig:ablation}(a) shows that slip can be detected almost immediately after onset: 43\% of trials are detected at the onset frame and 76\% within 20.8\,ms. Most long latency cases occur in the 2\,mm/s setting. This is the steady crosshead speed, while motion near slip onset is even slower. Even at 2\,mm/s, the object moves only 0.5\,mm during the 250\,ms observation window, close to the tactile array resolution limit, weakening both pressure migration and friction induced vibration. Consistent with this, detection latency decreases with slip speed (Appendix~\ref{app:speed}); at 18\,mm/s, over 90\% of trials can be detected within 20.8\,ms. Model inference averages 2.3\,ms on an NVIDIA RTX A4500 (measured over 1,000 runs), well within the 4.17\,ms frame budget at 240\,Hz. Combined with detection latency, 76\% of slip events are detected within 23.1\,ms overall.
  
\textbf{Choice of observation window T.} 
Fig.~\ref{fig:ablation}(b) shows Macro F1 as a function of the observation window $T$. Performance rises steeply from $T{=}6$\,fr (93.7\%) to $T{=}60$\,fr (96.7\%) and plateaus thereafter. Longer windows also increase variance across seeds, suggesting that excessive temporal context can dilute the local slip cues near the prediction frame. The saturation is physically grounded. At the slowest speed of 2\,mm/s, one taxel of displacement requires approximately 450\,ms, explaining the modest gains up to $T{=}120$\,fr (500\,ms). At higher slip speeds, $T{=}60$\,fr (250\,ms) already captures sufficient contact evolution, making longer windows largely redundant. 

\textbf{Piezo spatial resolution benefits.}
Fig.~\ref{fig:ablation}(c) shows the benefit of higher Piezo resolution under multimodal fusion. Lower resolutions ($16{\times}16$ and $8{\times}8$) are generated from native $32{\times}32$ measurements through spatial striding, while keeping the fusion pipeline unchanged. As expected, even an $8{\times}8$ array improves over XL-only, indicating that coarse spatial pressure already provides complementary information. Performance continues to improve with resolution, highlighting the value of fine-grained Piezo sensing and motivating the design towards higher resolution tactile arrays.

\begin{figure}[t]
  \centering
  \includegraphics[width=1.0\linewidth]{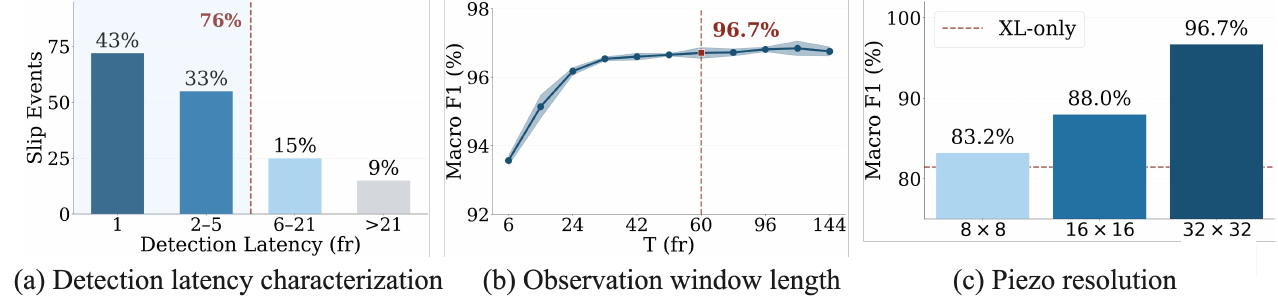}
  \caption{Results on UMI test splits. (a)~76\% slip events can be detected within 5 frames (20.8\,ms detection latency; 23.1\,ms including model inference). (b)~Performance saturates at $T{=}60$\,fr (250\,ms). (c)~Fusion performance consistently improves with Piezo resolution.}
  \label{fig:ablation}
\end{figure}

\subsection{Real-World Deployment and Slip Prevention Strategy}
To validate practical utility, we integrate \name~into a closed-loop re-grasp controller on the Tesollo hand, where \revised{four operators perform 100 trials on 10 unseen objects, including upward, downward, rotational, and oblique pulls at varied angles with two instructed speeds, to induce slip}. Upon slip detected, the fingers move inward to increase force until the object is secured. Although training data contains only vertical linear slip, \name~generalizes to varied slip directions and rotational slip, detecting all slip events in 100/100 trials without false alarms (Appendix~\ref{app:realworld}). These detections enable the controller to prevent object loss in 95/100 trials, with a detection-to-peak-force latency of $\sim$50\,ms. Importantly, 5 failures stem from the prevention strategy, not missed detections: when the cable is initially grasped near the sensor edge, the coarse inward motion pushes it out despite correct slip detection.

\section{Conclusions}\label{sec:conclusion}
We presented \name, a multimodal tactile slip detection framework that fuses $32{\times}32$ piezoresistive arrays with 8\,kHz accelerometers. Piezo and XL are complementary: Piezo grounds contact state and suppresses false alarms from environmental vibration, while XL captures friction transients during slip. Their fusion achieves 96.7\% Macro F1 with 76\% of slip events detected within 23.1\,ms. The learned model transfers zero-shot from a UMI parallel-jaw gripper to a Tesollo dexterous hand across unseen objects, distinct sensor units, and different contact configurations, suggesting that the model captures slip signatures intrinsic to the contact interface rather than platform specific artifacts. Integrated into a closed-loop re-grasp controller as a proof of concept, TacV5 and \name~ demonstrate the potential for fast tactile feedback in dexterous manipulation. Together, these results suggest a promising direction towards native integration of multimodal tactile in dexterous robots.

\section{Limitations} \label{sec:limitations}
 While \name\ demonstrates low-latency detection and zero-shot cross-platform transfer, several limitations remain. Data are collected under a single linear slip axis; rotational and multi-directional slip are successfully detected in real-world deployment (100/100 trials) but without quantitative latency characterization, and in-motion slip detection during active robot manipulation remains to be quantitatively validated as well. The piezoresistive array measures only normal pressure and infers shear implicitly from temporal frame sequences; direct shear sensing may provide further discriminative signal. Although evaluated on 37 objects, broader object and environmental coverage are natural next steps. Finally, inference is benchmarked on a GPU workstation; deployment on embedded hardware would enable a fully self-contained slip detection module.
  
\clearpage
\acknowledgments{\revised{We thank Greg Freeburn for support with equipment procurement, the Mark-10 package, and for setting up the data collection environment; Zachary Corriveau for support with the UMI platform and lab setup; and Jorge Alejandro and Michael Morganto for timely sensor support and helpful discussions on XL.
}}

\bibliography{references}
\clearpage
\appendix
\section*{Appendix}
This appendix provides supplementary material in \revised{nine} parts:
customized UMI setup (App.~\ref{app:umi}), XL transformation and encoder details (App.~\ref{app:xl_encoder}), experimental setup (App.~\ref{app:exp}), ablations on encoder architectures and inference cost (App.~\ref{app:encoder}), \revised{ablation on XL augmentation (App.~\ref{app:ablation_aug}), ablation on fusion strategy (App.~\ref{app:ablation_fusion}),} detailed detection latency characterization by slip speed (App.~\ref{app:speed}), real-world \name~deployment on UMI and Tesollo (App.~\ref{app:realworld}), and proof of direction invariance of merged log-mel spectrogram (App.~\ref{app:xl_invariance}).

\section{Customized UMI}\label{app:umi}
\begin{figure}[h!]
  \centering
  \includegraphics[width=0.7\linewidth]{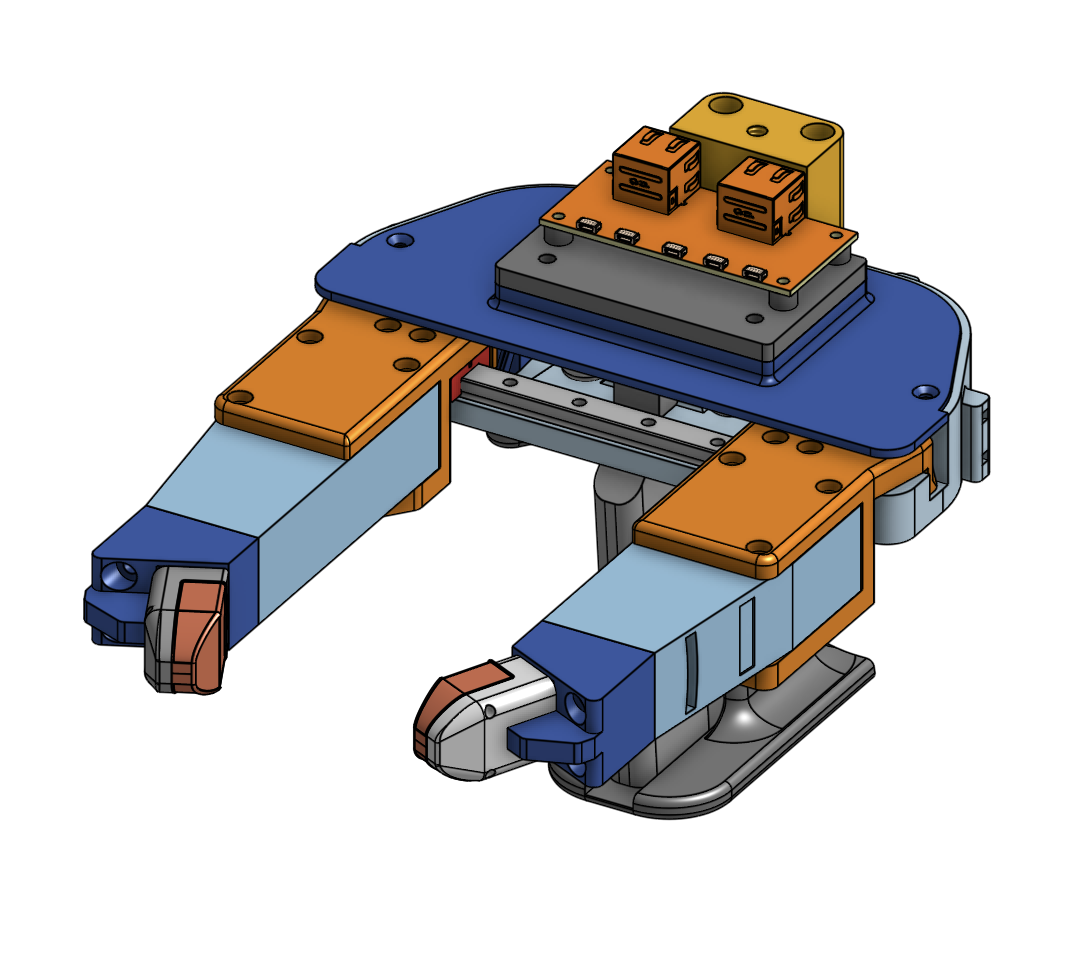}
  \caption{
A customized UMI was developed to augment the data-collection interface with tactile sensing. Two redesigned finger-adapter assemblies were mounted to the UMI body, enabling attachment of dexterous-hand fingertips equipped with tactile sensors, shown in orange. A top-mounted aggregation board, also shown in orange, was installed at the center of the main body to consolidate and synchronize data streams from the two tactile fingers.
  }
  \label{fig:umi}
\end{figure}

\section{Accelerometer Transformation and Encoder Details}\label{app:xl_encoder}
For each axis $k \in \{x,y,z\}$, we compute a causal log-mel spectral frame $\mathbf{M}_{t,k} \in \mathbb{R}^{F}$ from a 16\,ms analysis window ending at frame $t$, with a frame shift of $1/240$\,s aligned one-to-one with each Piezo frame. Each spectral frame uses only 
past and current observations without future information.                 
Contact vibrations project differently onto each accelerometer axis depending on slip direction, making per-axis spectrograms direction-dependent. While this property can be useful for slip direction estimation~\cite{a_slip_2026}, our data are collected under a fixed direction and we aim for direction-invariant generalization. We therefore aggregate across axes using log-sum-exp:
\begin{equation}                                               \mathbf{M}_{t}                             
  = \log\!\sum_{k \in \{x,y,z\}} \exp\!\left(\mathbf{M}_{t,k}\right) \in \mathbb{R}^{F},
\label{eq:logsum}
\end{equation}
This yields a provably direction-invariant representation (Appendix~\ref{app:xl_invariance}). 

The merged spectrogram is encoded by SSAST-Tiny~\cite{gong2022ssast}, an AST-family model~\cite{gong2021ast}. Contact signals exhibit sparse spectral patterns compared to speech, so we reduce the standard 128 mel bins to $F{=}16$, using non-overlapping $16{\times}1$ patches that collapse the frequency dimension to a single embedding per time step. The patch-embedding projection is reinitialized to accommodate the reduced input resolution, while all transformer layers retain pretrained weights. The encoder outputs a per-frame
embedding $\mathbf{z}^{xl} \in \mathbb{R}^{128}$.

\section{Experimental Details}\label{app:exp}
\subsection{Dataset}\label{app:dataset}
\begin{figure}[h!]
  \centering
  \includegraphics[width=0.9\linewidth]{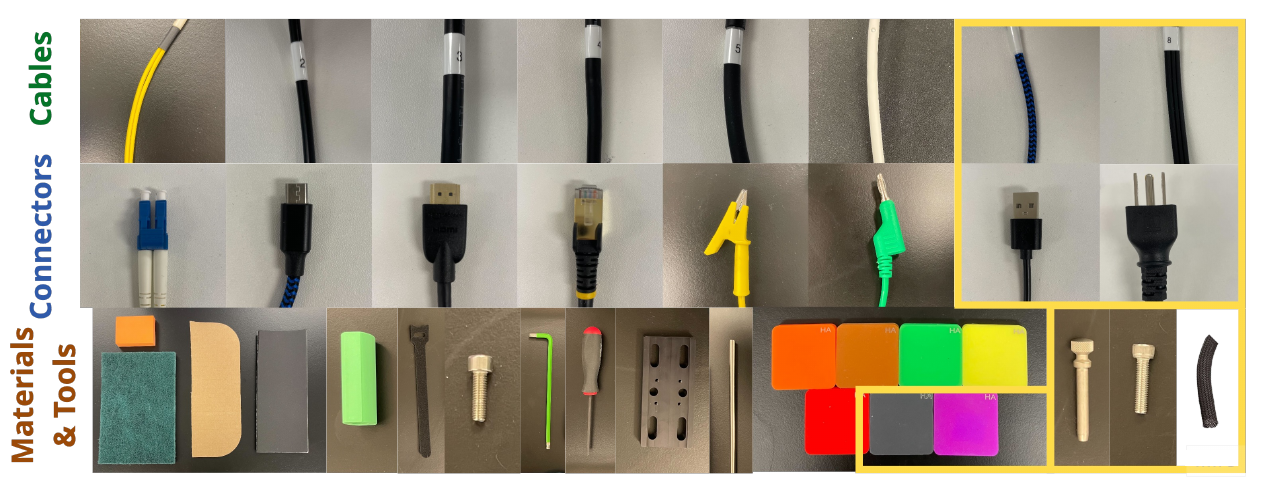}
  \caption{
    37 objects spanning cables, connectors, materials, and tools. We train on 28 objects and reserve 9 held-out objects for evaluation (yellow boxes). The held-out set further includes Shore hardness test blocks, with hardness levels 81\,HA and 88\,HA unseen from training to evaluate generalization across material properties.
  }
  \label{fig:object}
\end{figure}

Slip data were collected under five crosshead speeds ranging from 2--18\,mm/s and a human natural range of grasp forces. Fig.~\ref{fig:coverage} shows the distribution of total contact pressure across the training and test sets. The training set is collected by a human operator who intentionally varies grasp force, naturally reflecting the range of pressures encountered in human-operated grasping. The resulting training distribution (blue) covers the full range of contact forces observed during testing (pink), with both concentrated below 200\,kPa$\cdot$taxels and a long tail extending beyond 400\,kPa$\cdot$taxels. This confirms that the model is not evaluated outside its trained pressure regime. 
\begin{figure}[h!]
  \centering
  \includegraphics[width=0.7\linewidth]{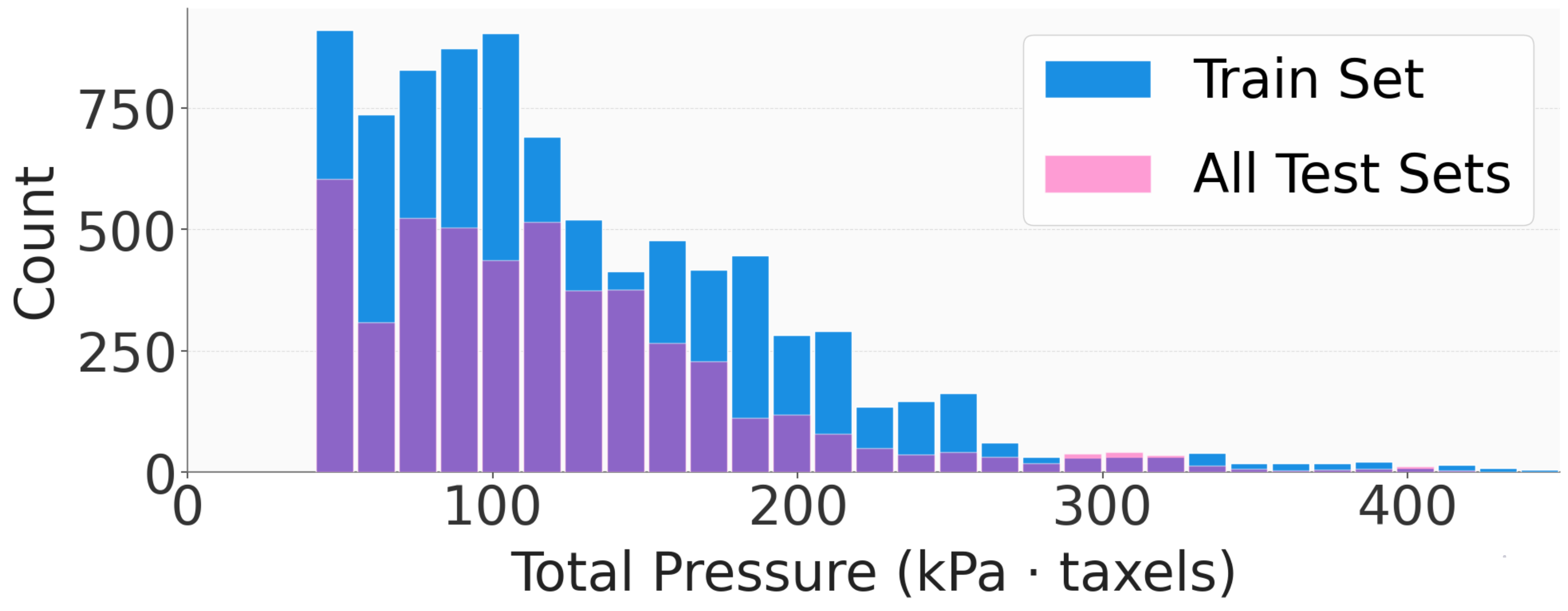}
  \caption{
    Contact pressure coverage. Distribution of total pressure (kPa$\cdot$taxels) for the training set (blue) and all test sets (pink). Training data is collected by a human operator intentionally varying grasp force, naturally reflecting the human-operated pressure range and covering the full test distribution.    
  }
  \label{fig:coverage}
\end{figure}

To reduce Mark10 induced artifacts and suppress false positives, we include three control scenarios in the training set, each with 5 trials:

(1) Crosshead motion with the UMI gripper not in contact with any object;

(2) Crosshead motion while the UMI gripper holds an object not linked to the crosshead;

(3) Free-grasping and pick-and-place motions using the UMI without induced slip.

\subsection{Implementation Details}\label{app:implementation}
We use a default observation window length of $T{=}60$. Piezo augmentation includes random rotation up to $360^\circ$, spatial rolling up to $\pm30\%$, horizontal and vertical flipping, and additive Gaussian noise with $\sigma=0.001$. XL log-mel spectrograms are standardized across frequency bins per frame and augmented with frequency masking of up to 5 bins applied with 50\% probability, together with additive uniform noise. All models are trained for 10 epochs with learning rate $1\times10^{-5}$ and batch size 32. Training uses a 5 epoch linear warmup followed by MultiStepLR decay with $\gamma=0.5$. We apply Exponential Moving Average for stable inference.

\subsection{Baseline Details}\label{app:baseline}
We compare against three analytic baselines: \textit{(1) Piezo-only}, adapting the entropy-based slip indicator from~\citet{hu2024gelsightmini} by replacing GelSight marker motion with frame-to-frame pressure differences $\Delta P_t$ on the $32\times32$ grid, restricted to active contact regions exceeding 5\% of peak pressure. \textit{(2) XL-only}, treating accelerometer signals as acoustic data and following~\citet{Yoo_2021} to compute 13 MFCC features over 26 mel filterbanks with Quadratic Discriminant Analysis (QDA) for classification. \textit{(3) Piezo-XL}, combining predictions from both baselines via logical AND. All thresholds are fitted on the training set and applied directly to test data.

\section{Ablation Study on Encoder Architectures}\label{app:encoder}
The choice of per-modality encoder directly affects how well each signal stream is represented before fusion. A weak encoder can bottleneck the entire pipeline regardless of fusion quality. Since single-modality performance is substantially lower than the fused model (Table~\ref{tab:modality_ablation}), ablating encoders in isolation would conflate encoder quality with the absence of the complementary modality. We therefore ablate in the full Piezo+XL fused setting, keeping the fusion pipeline and one modality's encoder fixed while varying the other. This isolates each encoder's contribution to the fused performance and better reflects the deployment configuration. The encoder candidates for each modality are described below.
  
\paragraph{Piezo Encoder (ResNet-18)}
ResNet-18~\cite{resnet} processes piezo tactile images (32×32) through successive convolutional blocks with skip connections, producing a 128-dimensional feature representation via global average pooling.

\paragraph{Piezo Encoder (ViT-MAE)}
Tactile images contain rich spatial structure that benefits from self-supervised
representation learning~\cite{zhu2025touch}. We pretrain a ViT-based encoder on approximately 2 million piezoresistive contact frames: {$\sim$}2M simulated grasps (72 objects) from our simulator across 72 objects, and {$\sim$}30K real-world frames from operator-held UMI grasping across 10 objects. We use a masked autoencoding (MAE)~\cite{he2022mae}, treating each pressure image $\mathbf{p}_{t,s} \in \mathbb{R}^{H \times W}$ as a single-channel image divided into non-overlapping $4{\times}4$ patches, yielding an $8{\times}8$ grid of 64 tokens. During pretraining, 60--80\% of patches are randomly masked in 95\% of training samples;
the asymmetric encoder (8 transformer layers, 128-dimensional hidden size, 8 attention heads)
processes only visible patches, while a lightweight decoder (4 layers, 64-dimensional hidden size,
4 attention heads) reconstructs masked patches from encoded visible tokens and learnable mask tokens.
Reconstruction loss (MSE on raw pixel values) is computed only on masked patches, encouraging spatially coherent representations from partial observations. Pretraining data is collected from both real grasps and physics-based simulation across diverse object geometries. After pretraining, masking is disabled and the encoder operates on all 64 patches to extract the \texttt{[CLS]} token as the force embedding $\mathbf{z}^{p} \in \mathbb{R}^{d}$, where d = 128.

Three variants are considered when training \name: \textit{i) scratch}, training the encoder from random initialization; \textit{ii) freeze}, keeping the pretrained parameters fixed; and \textit{iii) finetune}, initializing from pretrained weights and updating them during training.

\paragraph{Accelerometer Encoder (MLP)}
A compact two-layer multilayer perceptron for processing low-dimensional accelerometer features. The network transforms input vectors through a hidden layer with 16 units and ReLU activation, followed by dropout regularization, and outputs feature embeddings of configurable dimension. This architecture provides a parameter-efficient baseline for encoding pre-processed accelerometer signals.

\paragraph{Accelerometer Encoder (1D CNN)}
A causal convolutional network for temporal sequence modeling of accelerometer data. The architecture employs three 1D convolutional layers with batch normalization and ReLU activations: an initial temporal convolution with causal padding (left-padding only) to prevent information leakage from future timesteps, followed by two 1×1 convolutions. The network outputs per-frame feature embeddings, preserving the temporal structure of the input sequence.

Table~\ref{tab:architecture_ablation} reports results on UMI OOD. Execution time is reported as the mean over 1000 runs on an NVIDIA RTX A4500 GPU. For Piezo, all ViT-MAE variants outperform ResNet-18, confirming that spatial self-supervised pretraining benefits pressure-map encoding. The frozen encoder slightly outperforms finetuning, suggesting that pretrained representations already capture the relevant spatial structure and that end-to-end adaptation risks overfitting. The Custom CNN matches ViT-MAE variants across all metrics while being substantially lighter, indicating that $32{\times}32$ tactile images do not require the capacity of a vision transformer. For XL, the contrast is sharper. An MLP operating on flattened spectrograms collapses to 29-frame median latency, showing that temporal structure in the spectrogram is essential and cannot be captured by a position-agnostic architecture. The 1D CNN recovers most of the performance, but Tiny-SSAST achieves the best results across all metrics, validating the transfer of audio-domain pretrained weights to accelerometer spectrograms. 

\begin{table}[h!]
    \caption{Encoder architecture ablation.
    Each row varies one modality's encoder while fixing the other and the fusion pipeline (Piezo+XL). Evaluated on UMI OOD. Selected encoders ($\dagger$) achieve the best overall performance. 
    }
    \label{tab:architecture_ablation}
    \begin{center}
    \scalebox{0.85}{
    \centering
    \begin{tabular}{llcccc|c}
    \toprule
        Modality & Architecture &
        Lat (fr) $\downarrow$ & FPR (\%) $\downarrow$ & Acc (\%) $\uparrow$ & F1 (\%) $\uparrow$ & Execution Time (ms)$\downarrow$ \\
    \midrule
        \multirow{5}{*}{Piezo}
          & ResNet-18                   &  2.00 &  4.19 & 93.04 & 92.91 & 3.08 \\
          & ViT-MAE (scratch)           &  2.67 &  2.26 & 95.38 & 95.33 & 5.31 \\
          & ViT-MAE (freeze)            &  2.67 &  2.02 & 95.75 & 95.51 & 5.33\\
          & ViT-MAE (finetune)          &  2.00 &  3.17 & 95.07 & 95.03 & 5.32 \\
          & Custom CNN$^\dagger$        &  \textbf{1.50} &  \textbf{1.57} & \textbf{95.99} & \textbf{95.75} & 2.33 \\
    \midrule
        \multirow{3}{*}{XL}
          & MLP                         & 30.00 &  4.04 & 90.69 & 90.38 & 1.07 \\
          & 1D CNN                      &  2.50 &  2.30 & 94.88 & 94.68 & 1.33 \\
          & Tiny-SSAST$^\dagger$        &  \textbf{1.50} &  \textbf{1.57} & \textbf{95.99} & \textbf{95.75} & 2.33 \\
    \bottomrule
    \end{tabular}
    }
    \end{center}
\end{table}


\section{\revised{Ablation on XL Augmentation}}\label{app:ablation_aug}
\revised{Table~\ref{tab:ablation_aug} evaluates the effect of XL augmentation, i.e., frame-level standardization and spectral masking. Removing it substantially reduces Tesollo transfer performance while leaving UMI performance largely unchanged, suggesting that augmentation may help suppress platform-specific vibration signatures while preserving slip-relevant features.}

\begin{table}[h!]
\centering
\caption{\revised{Effect of XL augmentation.}}
\label{tab:ablation_aug}
    \scalebox{0.80}{
    \begin{tabular}{@{}lcccc|cccc@{}}
    \toprule
    & \multicolumn{4}{c}{UMI OOD} & \multicolumn{4}{c}{Tesollo 2-Finger} \\
    \cmidrule(lr){2-5} \cmidrule(lr){6-9}
    & Lat (fr) $\downarrow$ & FPR (\%) $\downarrow$ & Acc (\%) $\uparrow$ & F1 (\%) $\uparrow$ & Lat (fr) $\downarrow$ & FPR (\%) $\downarrow$ & Acc (\%) $\uparrow$ & F1 (\%) $\uparrow$ \\
    \midrule
    w/o XL aug             & 1.50 & 1.90 & 94.80 & 94.49 & 4.00 & 11.31 & 89.19 & 87.54 \\
    \midrule
    SlipSense              & \textbf{1.50} & \textbf{1.57} & \textbf{95.99} & \textbf{95.75} & \textbf{3.00} & \textbf{1.14} & \textbf{95.89} & \textbf{94.39} \\
    \bottomrule
    \end{tabular}}
\end{table}

\section{\revised{Ablation on Fusion Strategy}}\label{app:ablation_fusion}
\revised{Table~\ref{tab:ablation_fusion} compares cross-attention against simpler fusion strategies under identical training. Cross-attention consistently outperforms concatenation and late fusion across all metrics on both UMI OOD and Tesollo 2-Finger.}

\begin{table}[h!]
\centering
\caption{\revised{Comparison of fusion strategies.}}
\label{tab:ablation_fusion}
\scalebox{0.80}{
    \begin{tabular}{@{}lcccc|cccc@{}}
    \toprule
    & \multicolumn{4}{c}{UMI OOD} & \multicolumn{4}{c}{Tesollo 2-Finger} \\
    \cmidrule(lr){2-5} \cmidrule(lr){6-9}
    & Lat (fr) $\downarrow$ & FPR (\%) $\downarrow$ & Acc (\%) $\uparrow$ & F1 (\%) $\uparrow$ & Lat (fr) $\downarrow$ & FPR (\%) $\downarrow$ & Acc (\%) $\uparrow$ & F1 (\%) $\uparrow$ \\
    \midrule
    Late fusion            & 2.50 & 1.65 & 95.28 & 95.06 & 4.00 & 1.23 & 94.86 & 92.72 \\
    Concatenation          & 2.00 & 1.65 & 95.51 & 95.28 & 4.50 & 1.33 & 95.04 & 93.07 \\
    \midrule
    SlipSense (cross-attn) & \textbf{1.50} & \textbf{1.57} & \textbf{95.99} & \textbf{95.75} & \textbf{3.00} & \textbf{1.14} & \textbf{95.89} & \textbf{94.39} \\
    \bottomrule
    \end{tabular}}
\end{table}

\section{Detection Latency by Slip Speed}\label{app:speed}
Table~\ref{tab:latency_by_speed} and Fig.~\ref{fig:latency_cdf} break down detection latency by slip speed. Faster crosshead motion produces larger friction transients in XL and faster pressure redistribution in Piezo, both of which strengthen the slip signature. At 18\,mm/s, over 90\% of events are detected within 5 frames (20.8\,ms). Slow slip at 2\,mm/s is the most challenging condition, with a long tail extending to 105 frames, reflecting cases where gradual pressure migration takes time to accumulate sufficient evidence for detection.

\begin{table}[h!]                           
\caption{Detection Latency by Slip Speed. Median and mean detection delay (frames) across five programmed crosshead speeds on UMI test episodes. $n$ denotes the number of slip events per speed. One frame = 4.17\,ms at 240\,Hz.}
\label{tab:latency_by_speed}                                                \begin{center}                                                                       \scalebox{1.0}{                              
      \begin{tabular}{lccccc}                                                      \toprule                                     
Speed & $n$ &  Median (fr) $\downarrow$ & Mean (fr) $\downarrow$ & Max (fr) $\downarrow$ \\    
    \midrule                                                           
        2\,mm/s  & 36 & 3.0 & 14.9 & 106 \\                                     
        5\,mm/s  & 31 & 3.0 &  6.8 &  39 \\                                     
        10\,mm/s  & 37 & 1.0 &  3.4 &  27 \\                                    
        15\,mm/s  & 32 & 2.5 &  5.8 &  35 \\                                   
        18\,mm/s  & 31 & 1.0 &  2.9 &  20 \\                                    \bottomrule                                                            \end{tabular}}
\end{center}                                 
  \end{table} 

\begin{figure}[h!]
  \centering
  \includegraphics[width=0.5\linewidth]{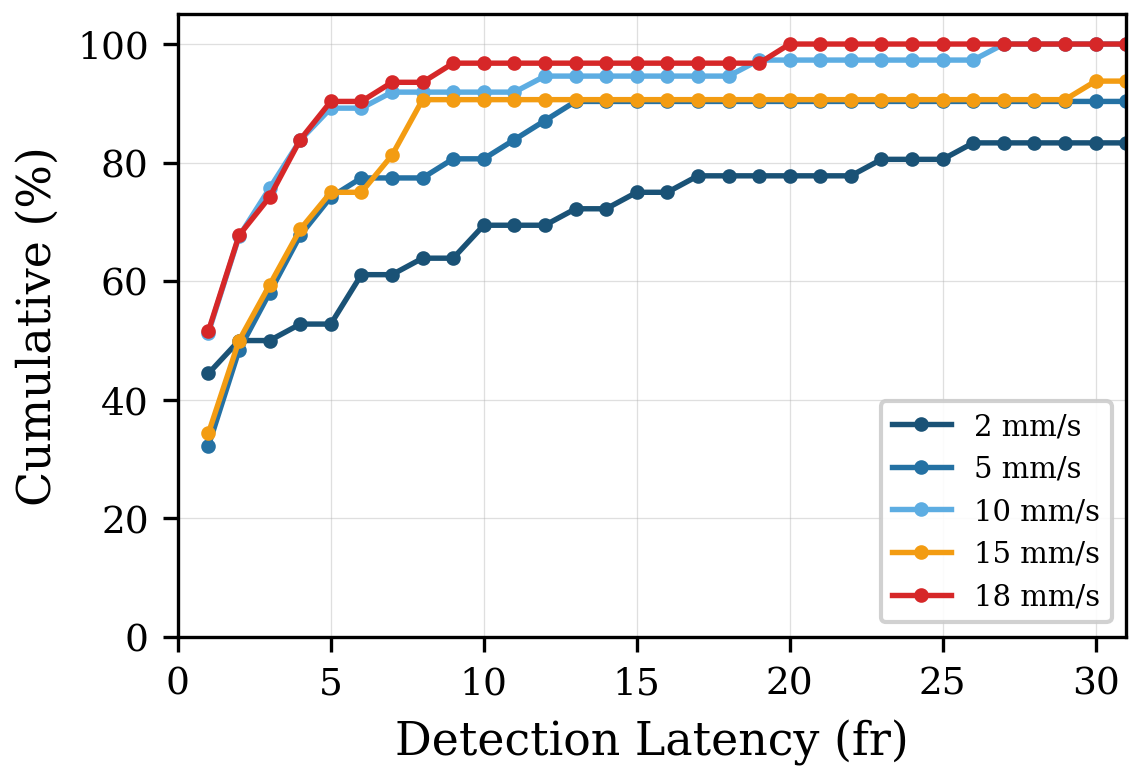}
  \caption{
    Detection latency CDF by slip speed. Cumulative distribution of slip-onset detection delay (in frames) across five crosshead speeds on UMI test episodes. Faster slip produces stronger friction transients and pressure migration, leading to earlier detection. At 18\,mm/s, over 90\% of events are detected within 5 frames (20.8\,ms).
  }
  \label{fig:latency_cdf}
\end{figure}

\section{Real-World Deployment}\label{app:realworld}
We validate \name~in two real-world settings (Fig.~\ref{fig:real_umi} and Fig.~\ref{fig:real_hand}). First, an operator holds a free-standing UMI gripper — as opposed to the table-mounted configuration used during data collection — introducing hand-induced motion that may excite the XL and varies Piezo contact pressure subtly. Despite these disturbances, \name~produces no false alarms, owing to the control scenarios in the training set. Second, we evaluate slip detection on the Tesollo hand across varied conditions, including different cable types, different pull directions, and rotational slip. \name~correctly identifies all slip events in 100 trials.

\begin{figure}[h!]
  \centering
  \includegraphics[width=0.8\linewidth]{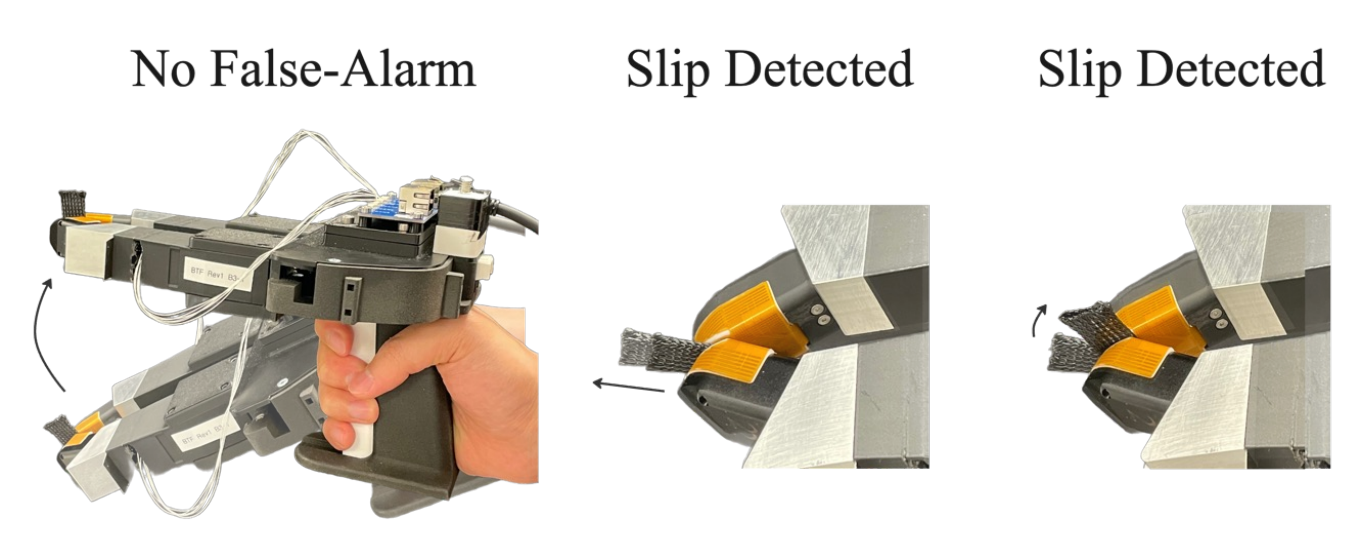}
  \caption{
    Real-world slip detection on UMI. Free UMI gripper held by an operator. Hand-induced motion may excite the XL and varies Piezo contact subtly, but \textit{\name} produces no false alarms (left). Translational slip (center) and rotational slip (right) can be correctly detected.
  }
  \label{fig:real_umi}
\end{figure}

\begin{figure}[h!]
  \centering
  \includegraphics[width=0.8\linewidth]{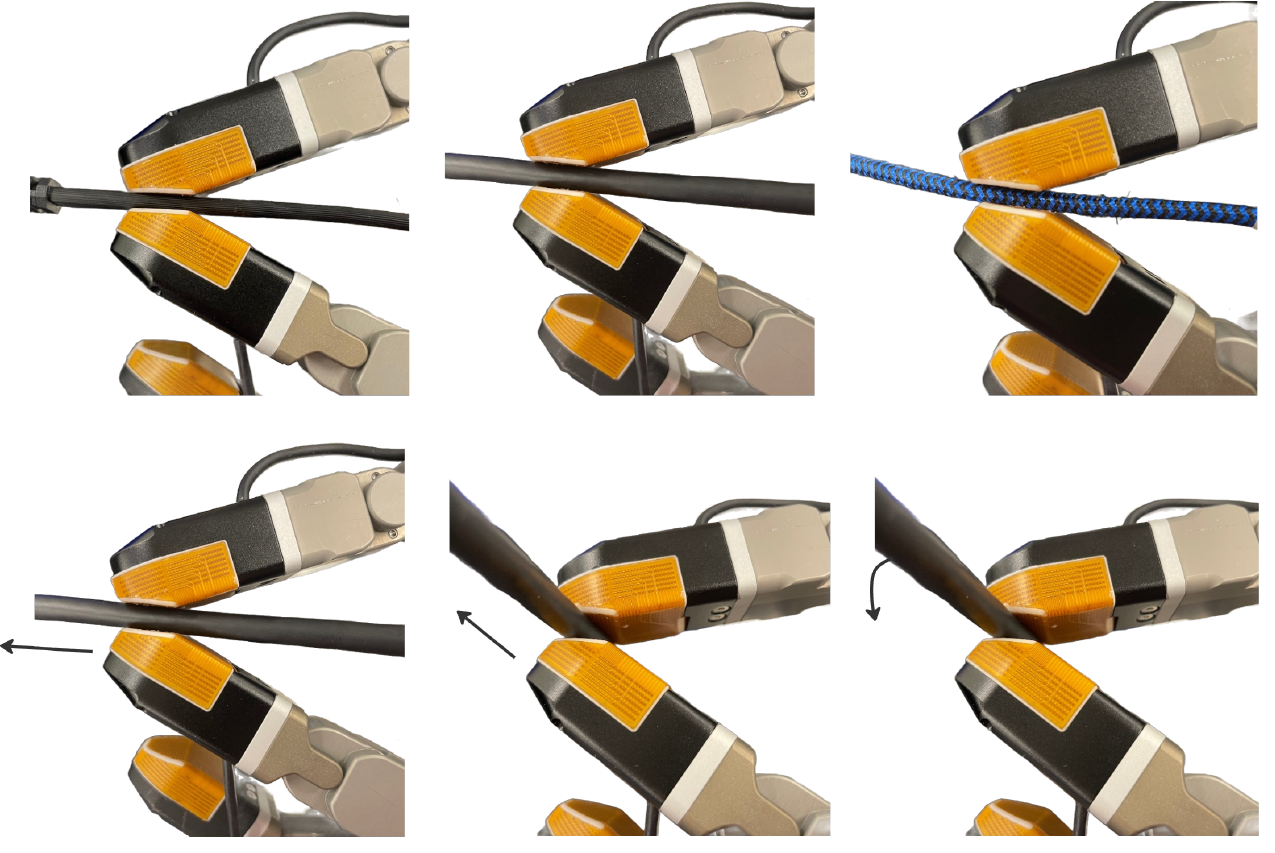}
  \caption{
    Real-world slip detection on the Tesollo hand. Training data contains only vertical linear slip. Through Piezo data augmentation (Appendix.~\ref{app:implementation}) and direction-invariant XL encoding (Appendix.~\ref{app:xl_invariance}), \textit{\name} generalizes to unseen directions and rotational slip. Top: hand grasp configurations on three representative cables; trials include additional types beyond those shown. Bottom: varied pull directions and rotational slip. \textit{\name} correctly identifies slip in all 100/100 trials.
  }
  \label{fig:real_hand}
\end{figure}

\clearpage
\section{Proof of Direction Invariance of Merged Log Mel Spectrogram} \label{app:xl_invariance}                    
  \paragraph{Notation.}                         
  Let $s(t) \in \mathbb{R}$ denote the scalar vibration signal at the contact
  interface, and let $\mathbf{d} = (d_x, d_y, d_z)^\top \in \mathbb{R}^3$ denote                                                                                                              
  the unit direction vector of slip or motion. The three-axis accelerometer
  measurements are $x(t)$, $y(t)$, $z(t) \in \mathbb{R}$, with corresponding
  STFT coefficients $X(f,\tau)$, $Y(f,\tau)$, $Z(f,\tau)$, $S(f,\tau) \in
  \mathbb{C}$. The $k$-th mel filterbank weight at frequency $f$ is $m_k(f) \geq
  0$, and $\mathcal{M}_k^s(\tau)$ denotes the $k$-th mel filterbank output applied
  to $|S(f,\tau)|^2$. The per-axis log-mel filterbank output is written
  $\mathcal{F}_k^{(i)}(\tau)$ for axis $i \in \{x, y, z\}$.

  \paragraph{Assumptions.}
  \begin{enumerate}
      \item \textbf{(Linear projection.)} The vibration at the contact interface
      is a scalar signal $s(t)$ propagating along a single fixed direction
      $\mathbf{d}$, such that the acceleration measured at the sensor satisfies:
      \begin{equation}
          x(t) = d_x \cdot s(t), \quad
          y(t) = d_y \cdot s(t), \quad
          z(t) = d_z \cdot s(t).
      \end{equation}

      \item \textbf{(Unit direction.)} $\mathbf{d}$ is a unit vector:
      \begin{equation}
          d_x^2 + d_y^2 + d_z^2 = 1.
      \end{equation}

      \item \textbf{(Quasi-stationary direction.)} $\mathbf{d}$ is constant
      within each STFT analysis window.

      \item \textbf{(Single vibration source.)} The sensor receives vibration
      from a single dominant source, such that superposition of independent
      multi-directional sources is negligible.
  \end{enumerate}

  \begin{proposition}
  Under assumptions \textbf{A1}--\textbf{A4}, the log-mel spectrogram computed
  from the summed power spectrum of the three accelerometer axes is invariant to
  the slip or motion direction $\mathbf{d}$.
  \end{proposition}

  \begin{proof}
  \textbf{Step 1: STFT linearity.}
  By linearity of the STFT and \textbf{A1}:
  \begin{equation}
      X(f,\tau) = d_x \cdot S(f,\tau), \quad
      Y(f,\tau) = d_y \cdot S(f,\tau), \quad
      Z(f,\tau) = d_z \cdot S(f,\tau).
  \end{equation}

  \textbf{Step 2: Per-axis power spectra.}
  \begin{equation}
      |X(f,\tau)|^2 = d_x^2\,|S(f,\tau)|^2, \quad
      |Y(f,\tau)|^2 = d_y^2\,|S(f,\tau)|^2, \quad
      |Z(f,\tau)|^2 = d_z^2\,|S(f,\tau)|^2.
  \end{equation}

  \textbf{Step 3: Summed power spectrum.}
  Summing across axes and applying \textbf{A2}:
  \begin{equation}
      |X(f,\tau)|^2 + |Y(f,\tau)|^2 + |Z(f,\tau)|^2
      = \underbrace{\left(d_x^2 + d_y^2 + d_z^2\right)}_{=\,1} |S(f,\tau)|^2
      = |S(f,\tau)|^2.
  \end{equation}
  The direction-dependent coefficients cancel exactly.

  \textbf{Step 4: Mel filterbank.}
  Applying the linear mel filterbank with non-negative weights to the summed
  power:
  \begin{equation}
      \mathcal{M}_k(\tau)
      = \sum_f m_k(f)\Bigl(|X|^2 + |Y|^2 + |Z|^2\Bigr)
      = \sum_f m_k(f)\,|S(f,\tau)|^2
      = \mathcal{M}_k^s(\tau).
  \end{equation}

  \textbf{Step 5: Log.}
  \begin{equation}
      \log \mathcal{M}_k(\tau) = \log \mathcal{M}_k^s(\tau),
  \end{equation}
  which is identical to the log mel spectrogram of $s(t)$ alone, independent of
  $\mathbf{d}$.
  \end{proof}

  \begin{corollary}[LogSumExp equivalence]
  When per-axis log-mel features are computed separately, the direction-invariant
  log-mel spectrogram is equivalently obtained via \emph{LogSumExp} across axes.
  Specifically, for each axis $i \in \{x, y, z\}$:
  \begin{equation}
      \mathcal{F}_k^{(i)}(\tau)
      = \log\!\left(d_i^2 \cdot \mathcal{M}_k^s(\tau)\right)
      = 2\log|d_i| + \log \mathcal{M}_k^s(\tau).
  \end{equation}
  Applying LogSumExp across axes:
  \begin{equation}
      \log \sum_{i \in \{x,y,z\}} \exp\!\left(\mathcal{F}_k^{(i)}(\tau)\right)
      = \log\!\left[\underbrace{\left(d_x^2 + d_y^2 + d_z^2\right)}_{=\,1}
        \mathcal{M}_k^s(\tau)\right]
      = \log \mathcal{M}_k^s(\tau). \qed
  \end{equation}
  Hence, LogSumExp applied across per-axis outputs is a computationally efficient and numerically stable realization of the direction-invariant log-mel spectrogram.
  \end{corollary}


\end{document}